\documentclass[journal,twoside,web]{ieeecolor}
\usepackage{generic}
\usepackage{cite}
\usepackage{amsmath,amssymb,amsfonts}
\usepackage{algorithmic}
\usepackage{graphicx}
\usepackage{textcomp}

\usepackage{bm}
\usepackage{bbm}
\usepackage{algorithm}
\newtheorem{theorem}{Theorem}
\let\mytheorem\thetheorem
\renewcommand\thetheorem{\arabic{section}.\mytheorem}
\newtheorem{lemma}{Lemma}
\let\mylemma\thelemma
\renewcommand\thelemma{\arabic{section}.\mylemma}
\newtheorem{corollary}{Corollary}
\let\mycorollary\thecorollary
\renewcommand\thecorollary{\arabic{section}.\mycorollary}

\def\BibTeX{{\rm B\kern-.05em{\sc i\kern-.025em b}\kern-.08em
    T\kern-.1667em\lower.7ex\hbox{E}\kern-.125emX}}
\begin{document}
\title{Asymptotic Convergence Rate of Alternating Minimization for Rank One Matrix Completion}
\author{Rui Liu and Alex Olshevsky%, \IEEEmembership{Member, IEEE}
\thanks{This work was supported by NSF project ECCN-1933027.}
\thanks{Rui Liu with the Division of Systems Engineering, Boston University, 
        {\tt\small rliu@bu.edu}}
\thanks{Alex Olshevsky with the Department of ECE and Division of Systems Engineering,  Boston University, 
        {\tt\small alexols@bu.edu}}}

\maketitle
\thispagestyle{empty}
\begin{abstract}
We study alternating minimization for matrix completion in the simplest possible setting: completing a rank-one matrix from a revealed subset of the entries. We bound the asymptotic convergence rate by the variational characterization of eigenvalues of a reversible consensus problem.  This leads to a polynomial upper bound on the asymptotic rate in terms of number of nodes as well as the largest degree of the graph of revealed entries.  
\end{abstract}


\section{INTRODUCTION}

\IEEEPARstart{M}{atrix} completion refers to the problem of completing a low rank matrix based on a subset of its entries. Algorithms for matrix completion have found many applications over the past decade, e.g., recommendation systems and the Netflix prize \cite{koren2009matrix,koren2009bellkor} or triangulation from incomplete data \cite{linial1995geometry, so2007theory}. However, despite much research into the topic, understanding the exact conditions under which matrix completion is possible remains open.

Formally, the problem may be stated as follows. Given a rank $r$ matrix $M \in \mathbb{R}^{n \times n}$, let $\mathcal{E} $ be a subset of indices such that the entries of $M$ corresponding to the subset $\mathcal{E}$ are revealed, i.e., we know the values of elements $M_{ij}$ if $(i,j) \in \mathcal{E}$. We use the standard notation $[n] = \{1, \ldots, n\}$ so that $\mathcal{E} \subset [n] \times [n]$. The goal is to find a rank $r$ matrix $M^e$ such that $P_{\mathcal{E}}(M^e)=P_{\mathcal{E}}(M)$, where $P_{\mathcal{E}}(\cdot)$ is defined as $\left[ P_{\mathcal{E}}(M) \right]_{ij}=\left\{\begin{matrix}
 M_{ij}&\text{ if } (i,j) \in \mathcal{E} \\ 
 0& \text{otherwise}
\end{matrix}\right..$

A popular approach is to solve matrix completion through convex relaxation \cite{candes2010power}. However, this approach is computationally expensive, and runs into difficulties for large scale systems, as each step of convex relaxation methods often need to truncate SVDs, which could take $O(n^3)$ operations. An alternative might be gradient descent on the Grassmann manifold \cite{keshavan2010matrix}, but to compute the gradient over the Grassmann manifold is also computationally intensive.

In contrast to this, alternating minimization algorithm is a cheap and empirically successful approach. Alternating minimization writes the low rank target matrix $M$ as $\alpha{\beta}^T$; then the algorithm alternates between finding the best $\alpha$ and the best $\beta$ to fit the revealed entries\cite{jain2013low}. It has been applied to clustering \cite{kim2008sparse}, sparse PCA \cite{zou2006sparse}, non-negative matrix factorization \cite{kim2008nonnegative}, signed network prediction \cite{hsieh2012low} among others. The main advantage of alternating minimization algorithm is the small size of the matrices one needs to keep track of (especially when the rank $r$ is much smaller than $n$) and smaller amount of computations. It was shown in \cite{jain2013low} that alternating minimization converges geometrically. However, the convergence time of alternating minimization is not fully understood, and its analysis often relies on either incoherence of the underlying matrix, random revealed pattern $\mathcal{E}$, or the assumption that the algorithm has a ``warm start'' -- or all of the above. 

Alternating minimization without any of these assumptions was considered in \cite{gamarnik2016note}, but only for the rank-one case. Although the problem of completing a rank-one matrix is trivial (indeed, one can find a factorization $M=xy^T$ just by recursively going through the revealed entries), it serves as the simplest possible setting where alternating minimization can be studied. Indeed, as we will see, a complete analysis even in this simple case remains open. 

It was established in \cite{gamarnik2016note} that if the graph corresponding to the revealed entries has bounded degrees and diameter which at most logarithmic in the size of the matrix, alternating minimization converges in polynomial time. In this paper, we are interested in studying the convergence time for rank-one matrices, but without any assumptions on degrees or diameter. 

Our main result is a polynomial time bound on the asymptotic convergence rate of the process, obtained by drawing on the connection to reversible consensus dynamics. An implication of our result is that the convergence time of alternating minimization to shrink the distance to the optimal solution by a factor of $\epsilon$ can be upper bounded by $O(n(n-1) \Delta \log (1/\epsilon))$ for all small enough $\epsilon$ (as a function of $n$ and the initial condition), where $\Delta$ is the largest degree in a graph corresponding to the revealed entries.  

A number of papers analyze algorithms for rank-one matrix completion. Gradient descent for sparse rank-one estimation of square symmetric matrices are studied in \cite{ma2018gradient}. Our paper analyzes rank-one matrix completion in general, not only for symmetric matrices, and using alternating minimization rather than gradient descent. A very recent work \cite{fattahi2020exact} is concerned with recovering the dominant non-negative principal components of a rank-one matrix precisely, where a number of measurements could be grossly corrupted with sparse and arbitrary large noise, a nice feature that we do not address in this work.  For recovering a rank-one matrix when a perturbed subset of its entries with good sensitivity to perturbations, two algorithms are presented in \cite{saligrama2020minimax}.

We begin by describing formally the main algorithm analyzed here.

\subsection{Alternating Minimization for Rank-one Matrix Completion Problem}

Consider a rank-one matrix $M=\alpha{\beta}^T$, where $\alpha, \beta \in \mathbb{R}^{n}$. Abusing notation slightly, let $\mathcal{V}_R$ and $\mathcal{V}_C$ denotes the sets of rows and columns of matrix $M$ respectively; of course, both $\mathcal{V}_R$ and $\mathcal{V}_C$ are equal to $[n]$, but writing $i \in \mathcal{V}_R$ vs $i \in \mathcal{V}_C$ will be convenient in terms of making it clear whether we are considering a row or a column $i$. 

We let $\mathcal{V}=\mathcal{V}_R \cup \mathcal{V}_C$ be the vertex set of the graph $\mathcal{G}=\{\mathcal{V},\mathcal{E}\}$, a bipartite undirected graph. The graph $\mathcal{G}$ has vertices corresponding to every row and column of the target matrix, and the edge $(i,j)$ is present in $\mathcal{G}$ precisely when the $(i,j)$'th entry of $M$ is revealed. 

For this bipartite graph $\mathcal{G}$, we write $i \sim j$ if $i \in \mathcal{V}_R$ is connected to $j \in \mathcal{V}_C$. We let $A$ be the adjacency matrix of $\mathcal{G}$, i.e., $A_{ij}=\begin{cases}
1 & \text{ if } i \sim j \\ 
0 & \text{ otherwise } 
\end{cases},$ and we denote by $\Delta$ the maximum degree and by $d$ the diameter of $\mathcal{G}$.

The matrix completion problem is the minimization problem $\underset{x,y \in \mathbb{R}^{n}}{\min} \sum_{(i,j) \in \mathcal{E}}(x_i y_j -M_{ij})^2.$ Note that the sum is taken over all the revealed entries; the goal is to find $x,y$ with a zero objective value. 

Our starting point is the alternating minimization method given in the box below, which was called Vertex Least Squares (VLS) in \cite{gamarnik2016note}. The convergence result proved in \cite{gamarnik2016note} is given in the subsequent theorem.

\begin{algorithm}[h]
\caption{Vertex Least Squares (VLS)}
\begin{algorithmic}[1]
\STATE For {$i \in \mathcal{V}_R$, $j \in \mathcal{V}_C$}, initialize $x_{i,0},y_{j,0}$\
\FOR {$t=1$ to $T$} 
    \FOR{$i \in \mathcal{V}_R$}
    \STATE \begin{equation}\label{xupdate}
        x_{i,t+1}=\arg \underset{x \in \mathbb{R}}{\min}  \sum_{j \in \mathcal{V}_C: i \sim j}(x^Ty_{j,t}-M_{ij})^2
    \end{equation}
    \ENDFOR
    \FOR{$j \in \mathcal{V}_C$}
    \STATE \begin{equation}\label{yupdate}
        y_{j,t+1}=\arg \underset{y \in \mathbb{R}}{\min}  \sum_{i \in \mathcal{V}_R: i \sim j}(y^Tx_{i,t+1}-M_{ij})^2
    \end{equation}
    \ENDFOR
\ENDFOR
\end{algorithmic}
\end{algorithm}

\begin{theorem}[Theorem 2.1 in \cite{gamarnik2016note}]\label{VLS}
Let $M=\alpha \beta^T$ with $\alpha, \beta \in \mathbb{R}^n$ and suppose the following assumptions hold:

(a) There exists $0<b<1$ such that for all $i,j \in [n]$, we have $b \leq \alpha_i,\beta_j \leq 1/b$.

(b) The graph $\mathcal{G}$ is connected.

(c) The graph $\mathcal{G}$ has diameter $d \leq c\log{n}$ for some fixed constant $c$, and maximum degree $\Delta$.

Then, there exists a constant $a>0$ which depends on $c$, $\Delta$ and $b$ only, such that for any initialization $b \leq x_{i,0},y_{j,0} \leq 1/b$, $i \in [n]$ and $\epsilon >0$, there exists an iteration number $T=O(n^{a}\log{n})$ such that after $T$ iterations of VLS, we have $\frac{1}{n}\|x_T y_T^T-M\|_F<\epsilon.$
\end{theorem}

\subsection{Our Contributions and Outline}

Theorem \ref{VLS} gives a polynomial convergence time, but under relatively strong assumption on degrees and diameter. In this paper, we want to remove this assumption; however, this will be obtained at the cost of obtaining bounds on the asymptotic convergence rate instead.
 
We begin with Section 2, where we define the asymptotic convergence rate of the VLS and connect it to a reversible consensus problem; we then exploit this to obtain a quadratic optimization problem which bounds this convergence rate. In Section 3, we prove combinatorial bounds for the convergence rate implied by this connection. Finally, Section 4, numerically evaluates our bounds on various classes of graphs via simulation, and Section 5 contains some brief concluding remarks. 

\section{The Connection to Reversible Consensus}
In this section, we begin by describing a connection to the consensus problem made in \cite{gamarnik2016note}. We then show that the resulting consensus problem is reversible, which allows us to write down a variational characterization of the convergence rate. 

The following steps follow the proof of Theorem 1 in \cite{gamarnik2016note}. From the update rules for VLS in (\ref{xupdate}) and (\ref{yupdate}), we have
\begin{equation} \label{x&y}
    x_{i,t+1}=\frac{\sum_{j:i \sim j}M_{ij}y_{j,t}}{\sum_{j:i \sim j}y_{j,t}^2} \text{ and } y_{j,t+1}=\frac{\sum_{i:i \sim j}M_{ij}x_{i,t+1}}{\sum_{i:i \sim j}x_{i,t+1}^2}.
\end{equation}
Define 
\begin{equation}\label{defu&v}
    u_{i,t}=\frac{x_{i,t}}{\alpha_i} \text{  and  } v_{j,t}=\frac{y_{j,t}}{\beta_j},
\end{equation} where, recall, $M=\alpha \beta^T$. By using (\ref{x&y}), the updates for $u_{i,t}$ can be written as:
\begin{equation}\label{uup}
\begin{split}
    u_{i,t+1}&=\frac{x_{i,t+1}}{\alpha_i}=\frac{\sum_{j:i \sim j}M_{ij}y_{j,t}}{(\sum_{j:i \sim j}y_{j,t}^2)\alpha_i}=\frac{\sum_{j:i \sim j}\beta_j y_{j,t}}{\sum_{j:i \sim j}y_{j,t}^2}\\
    &=\frac{\sum_{j:i \sim j}{y_{j,t}^2}/{v_{j,t}}}{\sum_{j:i \sim j}y_{j,t}^2}=\sum_{j:i \sim j}\frac{y_{j,t}^2}{\sum_{k:i \sim k}y_{k,t}^2}\frac{1}{v_{j,t}}.
\end{split}
\end{equation}
Similarly, we have 
\begin{equation}\label{vup}
    \frac{1}{v_{j,t}}=\sum_{i:i \sim j}\frac{\alpha_ix_{i,t}}{\sum_{k:k \sim j}\alpha_k x_{k,t}}{u_{i,t}}.
\end{equation}
Note that $u_{i,t+1}$ can be expressed as a convex combination of $\left\{ \frac{1}{v_{1,t}},\frac{1}{v_{2,t}},\cdots,\frac{1}{v_{n,t}} \right\}$ and $\frac{1}{v_{j,t}}$ can be expressed as a convex combination of $\{ u_{1,t},u_{2,t},\cdots,u_{n,t} \}$ for all $i,j \in [n]$. Rewriting (\ref{uup}) and (\ref{vup}) in the compact form:
\begin{equation} \label{u&v}
    u_{t+1}=B_t\left({\frac{1}{v}}\right)_t \text{ and } \left({\frac{1}{v}}\right)_t=C_t u_t \quad \forall t\geq 0,
\end{equation}
where $B_t=(b_{ij,t})_{n \times n}$ and $C_t=(c_{ij,t})_{n \times n}$ are $n \times n$ stochastic matrices, and $b_{ij,t}=\frac{y_{j,t}^2\mathbbm{1}(A_{ij}=1)}{\sum_{k:i \sim k}y_{k,t}^2} \text{, }c_{ij,t}=\frac{\alpha_j x_{j,t}\mathbbm{1}(A_{ji}=1)}{\sum_{k:k \sim i}\alpha_k x_{k,t}}.$ Combining the two updates in (\ref{u&v}), it follows that
\begin{equation}\label{udynamic}
    u_{t+1}=P_t u_t,
\end{equation}{}where $P_t=B_t C_t$ is also a stochastic matrix. Thus alternating minimization in this context can be written in terms of a consensus iteration. This concludes our summary of the  connection between rank-one alternating minimization and consensus which was discovered in \cite{gamarnik2016note}. 

We now begin our analysis by observing that the matrices $P_t$ appearing above correspond to reversible Markov chains. 

\begin{lemma}\label{rev} The Markov chain with transition probability matrix $P_t$ is reversible for all $ t \geq 0$. 
\end{lemma}

\begin{proof} Recall that a Markov chain with transition matrix $P$ and invariant measure $\pi=(\pi_1,\cdots,\pi_n)$ is reversible if and only if $\pi_i p_{ij}=\pi_j p_{ji}$ for all $i$ and $j$. 
 
 We have that $P_t=B_t C_t$, where
 \begin{equation}\label{pdef}
      p_{ij,t}= \frac{\alpha_j x_{j,t}}{\sum_{k:i \sim k}y_{k,t}^2} \sum_{l=1}^n \frac{y_{l,t}^2 \mathbbm{1}((i,l) \in \mathcal{E}) \mathbbm{1}((j,l) \in \mathcal{E})}{\sum_{k:k \sim l}x_{k,t}\alpha_k}.
 \end{equation}
Letting
\begin{equation}\label{pidef}
    \hat{\pi}_{i,t}=\alpha_i x_{i,t} \sum_{k:i \sim k} y_{k,t}^2,
\end{equation}{}it is now immediate that $\hat{\pi}_{i,t} p_{ij,t}=\hat{\pi}_{j,t} p_{ji,t}$. Normalizing $\pi_{i,t}=\frac{\hat{\pi}_{i,t}}{\sum_{i=1}^n \hat{\pi}_{i,t}}$,  we have a stochastic vector $\pi_t=(\pi_{1,t},\cdots,\pi_{n,t})$ such that $\pi_{i,t} p_{ij,t}=\pi_{j,t} p_{ji,t}$ for each $i,j \in [n]$ and $t>0$.
\end{proof}

The previous lemma will be the launching point of our analysis. To analyze the asymptotic convergence rate, we need to look at the limit of the matrices $P_t$ as $t \rightarrow \infty$; for this, we need to assert that $x_t, y_t$ converge, which we do in the following lemma. 

\begin{lemma} Under the assumptions that (i) $\mathcal{G}$ is connected, (ii) $b \leq \alpha_i, \beta_j \leq b^{-1}$, (iii) $b \leq x_{i,0}, y_{j,0} \leq b^{-1}$, we have that $u_t$ approaches a point in ${\rm span}\{\bf 1\}$ and the the cost $\sum_{(i,j) \in \mathcal{E}} (x_{i,t} y_{j,t} - M_{ij})^2$ approaches zero. \label{lemma:conv}
\end{lemma}

\begin{proof} The proof follows straightforwardly from the observation that the positive entries of the matrices $P_t$ constructed above are bounded away from zero. Indeed, by definition we have $b^2 \leq u_{i,0},v_{j,0} \leq \frac{1}{b^2}$ for all $i,j \in [n]$. Then since the updates of (\ref{uup}) and (\ref{vup}) are convex combinations, we get $b^2 \leq u_{i,t},v_{j,t} \leq \frac{1}{b^2}$ for all $t\geq 0$, $i,j \in [n]$. From the definition of $u_{i,t}$ and $v_{j,t}$ in (\ref{defu&v}), we can conclude that $b^3 \leq x_{i,t},y_{j,t} \leq \frac{1}{b^3}$ for all $t\geq 0$, $i,j \in [n]$. Putting this together with the expression for $p_{ij,t}$ from (\ref{pdef}), we obtain that the positive entries of $P_t$ are uniformly bounded below. Furthermore, since we can always take $i=j$ in (\ref{pdef}), we see that every $P_t$ has positive diagonal. Standard consensus theory (e.g., Theorem 1 in \cite{blond}) gives that $u_t$ converges to a multiple of the all-ones vector. 

If $u_t \rightarrow c {\bf 1}$, then equation (\ref{u&v}) implies that $v_t \rightarrow c^{-1} {\bf 1}$. Thus if 
$x_t \rightarrow c \alpha$, then $y_t \rightarrow c^{-1} \beta$. Thus $x_t y_t^T$ approaches $\alpha \beta^T$ and the cost approaches zero. 
\end{proof} 

Our next lemma collects some properties of the matrix $P$ which is the limit of the matrices $P_t$ as $t \rightarrow \infty$. 

\begin{lemma} Under the conditions of Lemma \ref{lemma:conv}, we have that $P$ is a stochastic matrix corresponding to a reversible Markov chain. It has real eigenvalues $1=\lambda_1,\lambda_2,\cdots,\lambda_n$, listed in order of decreasing magnitude. Furthermore, $\lambda_n > -1$. Moreover, if $\pi$ is the stationary distribution of $P$ and $z$ is any non-principal eigenvector (i.e., not corresponding to eigenvalue $1$), then $\pi^T z = 0$. Finally, we have that  
 $ \rho(P-\bm{1}\pi)=\max\{\lambda_2(P),-\lambda_n(P)\},$ where $\rho(\cdot)$ denotes the spectral radius. \label{lemma:properties}
\end{lemma} 

\begin{proof} That $P$ is a stochastic matrix corresponding to a reversible Markov chain follows that it is the limit of $P_t$, and, as we showed in Lemma \ref{rev}, each $P_t$ has these properties, and one is a simple eigenvalue of $P$.  The reversibility condition $\pi_i P_{ij} = \pi_j P_{ji}$ may be written as
$ {\rm diag}(\pi) P = P^T {\rm diag}(\pi). $ An implication of this is that $P$ is self-adjoint in the inner product $\langle x, y \rangle_{\pi} = \sum_{i=1}^n \pi_i x_i y_i$. Thus the eigenvalues of $P$ are real, and the eigenvectors of $P$ are orthogonal in this inner product. Since the top eigenvector is ${\bf 1}$, this implies that for any other eigenvector $z$, we have $0 = \langle {\bf 1}, z \rangle_{\pi} = \pi^T z.$

Moreover, that $1$ is the largest eigenvalue of $P$ follows (for any stochastic matrix) from  the Perron-Frobenius theorem. That the smallest eigenvalue is strictly above $-1$ follows by Gershgorin circles as  $P$ is a stochastic matrix with positive diagonal.

Finally, let $z_1=\bm{1},z_2,\cdots,z_n$ be the  eigenvectors of $P$ corresponding to eigenvalues $\lambda_1=1,\lambda_2,\cdots,\lambda_n$. Then $z_1=\bm{1},z_2,\cdots,z_n$ are also eigenvectors of $P-\bm{1}\pi$ corresponding to eigenvalues $0,\lambda_2,\cdots,\lambda_n$, since $\pi^T z_i = 0$ for any $i \geq 2$, we have that
\begin{align} 
(P-\bm{1}\pi)z_1 & = \bm{0},& \nonumber \\ 
(P-\bm{1}\pi)z_i&=P z_i-\bm{1}\pi z_i = \lambda_i z_i& \forall i \geq 2. \label{eq:eig}
\end{align} It follows that $\rho(P-{\bf 1} \pi) = \max(\lambda_2,-\lambda_n)$. 
\end{proof} 

We define the asymptotic convergence rate as $\gamma_{\rm asym}=\sup_{u_0 \notin U} \lim_{t \rightarrow \infty} \left (\frac{\|u_t-u^*\|_2}{\|u_0-u^*\|_2}\right)^{1/t},$ where $U=\{c\bm{1}:c \in \mathbb{R}\}$ and $u^*=\lim_{t \rightarrow \infty} u_t$. Naturally, the quantity $\gamma_{\rm asym}$ is related to the matrix $P$, and the following lemma makes a precise statement of this.  

\begin{lemma}\label{ga-lam} Under the conditions of Lemma \ref{lemma:conv}, we have that  
 $\gamma_{\rm asym} \leq \rho(P-\bm{1}\pi)$ where $\rho(\cdot)$ denotes the spectral radius.
\end{lemma}{}

\begin{proof} Observe that for any stochastic matrix $Q$, $ (I - {\bf 1} \pi) Q (I - {\bf 1} \pi)=(I - {\bf 1} \pi) Q.$

As a consequence, if we define for any stochastic matrix $Q$, the matrix $Q'$ as  $Q' = (I - {\bf 1} \pi) Q$, we have that $ Q' (I - {\bf 1} \pi) = (I - {\bf 1} \pi) Q. $ Then
$(I - {\bf 1} \pi) P_t \cdots P_1 u_0  =  P_t' \cdots P_1' (I - {\bf 1} \pi) u_0,$ and therefore $||u_t - \pi u_t {\bf 1}||_2 \leq ||P_t' \cdots P_1'||_2 ||(I - {\bf 1} \pi) u_0||_2.$ Since $u^*=c {\bf 1}$ where $c$ lies in the convex combination of the entries of $u_0$, we have that 
\begin{align*}  ||u_t - u^*||_2 & \leq   \sqrt{n} ||u_t - u^*||_{\infty}
 \leq  2 \sqrt{n}  ||u_t - \pi u_t {\bf 1}||_{\infty}& \\ 
& \leq   2 \sqrt{n} ||P_t' \cdots P_1'||_2  ||( I - {\bf 1} \pi) u_0 ||_2.&
\end{align*} Therefore 
\begin{align*}
    \gamma_{\rm asym} &= \sup_{u_0 \notin U} \lim \sup_{t \rightarrow \infty} \left( \frac{||u_t - u^*||_2}{||u_0 - u^*||_2}\right)^{1/t} \\
    & \leq \lim \sup_t ||P_t' \cdots P_1'||_2^{1/t}.
\end{align*}

Next, observe that we can repeat the same argument but beginning at iteration $k$ rather than iteration $1$. That is:
\begin{align*} \gamma_{\rm asym} &  \leq   \sup_{u_k \notin U} \lim \sup_t \left( \frac{||u_t - u^*||_2}{||u_k - u^*||_2}\right)^{1/(t-k)}& \\ 
 & \leq   \lim \sup_t ||P_t' P_{t-1}' \cdots P_k'||^{1/(t-k)}.&
\end{align*}
In particular, we have that for every $k$,
\begin{equation} \label{gammaineq} \gamma_{\rm asym} \leq \rho(\{ P_k', P_{k+1}', \ldots \}), 
\end{equation} where $\rho(\mathcal{M})$ is the joint spectral radius of the matrix set $\mathcal{M}$, defined as $ \rho(\mathcal{M}) = \lim_{m \rightarrow \infty} \sup ||\Pi_m||_2^{1/m},$ where the supreme is taken over all products $\Pi_m$ of $m$ matrices from the set $\mathcal{M}$. We refer the reader to \cite{jungers2009joint} for background on the joint spectral radius. In particular, by Lemma 1.2 of \cite{jungers2009joint}, we have that for any bounded set $\mathcal{M}$, we have that for any fixed $m$, $ \rho(\mathcal{M}) \leq  \sup ||\Pi_m||_2^{1/m}.$

We next argue that the right-hand side of (\ref{gammaineq}) can be bounded by $\rho(P')$ as $k \rightarrow \infty$. Indeed, for any $\epsilon > 0$, by definition of joint spectral radius there is a large enough integer $m$ so that 
$ \rho(\{P'\}) > \max ||\Pi_m||^{1/m} - \epsilon,$ keeping in mind that the $\max$ on the right-hand side is over a single product, namely $(P')^m$. Next, choose $r$ small enough so that if $\mathcal{M}$ is taken to be the ball of radius $r$ around $P'$, then the right-hand side of the last inequality only changes by $\epsilon$. Finally, choose $k$ large enough so that every $P_t', t\geq k$ lies in a ball of radius $r$ around $P'$. Putting all this together, we have that 
\begin{align*} \rho(\{P_k', P_{k+1}',\ldots \}) \leq & \rho(B_r(P'))  \leq \max_{\mathcal{M}.= B_r(P')} ||\Pi_m||^{1/m} \\ 
\leq & \max_{\mathcal{M}.= P'} ||\Pi_m||^{1/m} + \epsilon \leq  \rho(\{P'\}) + 2 \epsilon.
\end{align*}  Since $\epsilon$ is arbitrary, we have that $\lim \underset{k \rightarrow \infty}{\sup} \rho(\{P_k', P_{k+1}',\ldots \})\\  \leq \rho(P')$. Putting this together with (\ref{gammaineq}), we conclude that $\gamma_{\rm asym} \leq \rho(P')$ as desired. 

Finally, we note that because $\pi$ is a left-eigenvector of $P$ with eigenvalue $1$, we have that $P'=P - {\bf 1} \pi$. Thus $\gamma_{\rm asym} \leq \rho(P - {\bf 1} \pi)$ and the proof is concluded. 
\end{proof}{}

We conclude this section by putting together all the previous lemmas to obtain a variational upper bound on the convergence rate. 

\begin{corollary} Under the conditions of Lemma \ref{lemma:conv}, we have that $$\gamma_{\rm asym} \leq \max\{ 1-\frac{1}{2} \min_{x \in S} \sum_{i=1}^n \sum_{j=1}^n \pi_i p_{ij} (x_i-x_j)^2,-\lambda_n(P)\},$$ where $S=\{x|\sum_{i=1}^n \pi_ix_i=0, \sum_{i=1}^n \pi_i x_i^2=1\}.$
\label{cor:maincor}
\end{corollary}

\begin{proof} Lemma \ref{ga-lam} shows that $\gamma_{\rm asym} \leq \max \{\lambda_2(P) , \\-\lambda_n(P) \}$. This corollary simply replaces $\lambda_2(P)$ with its variational characterization; these were proved for reversible $P$ in \cite{olshevsky2011convergence}.
\end{proof}

We note that it is also possible to replace $-\lambda_n(P)$ by the variational characterization of it, but we will just leave it as $-\lambda_n(P)$ above, as it turns out that there are easy ways to bound it.

\section{An upper bound on the convergence rate}

We can use the main result of the previous section, namely Corollary \ref{cor:maincor}, to obtain an upper bound on the asymptotic convergence rate associated with alternating minimization. This is given in the following theorem, which is our main result.  

\begin{theorem}\label{Plambda} Suppose the assumptions of Lemma \ref{lemma:conv} hold, and additionally the graph $\mathcal{G}$ has maximum degree $\Delta$. Then, $$
    \gamma_{\rm asym} < 1-\frac{b^{12}}{n(n-1)\Delta}< 1-\frac{b^{12}}{n^3}.$$
\end{theorem}

\begin{proof} Glancing at Corollary \ref{cor:maincor}, we see that we need to bound the variational characterizations of $\lambda_2(P)$ in the statement of the corollary, as well as $-\lambda_n(P$). Our first step is to analyze the variational expression for $\lambda_2(P)$ in that corollary.

Note that the support of $P$ is same as the support of $AA^T$; indeed, 
 $P$ is the transition probability matrix of a certain random walk on $(\mathcal{V}_R,\mathcal{E}_R) \overset{\Delta}{=}\mathcal{G}_P$, where $(i_1,i_2) \in \mathcal{E}_R$ if and only if $i_2$ is a distance two neighbor of $i_1$ in $\mathcal{G}$. This new graph $\mathcal{G}_P$ is connected because, by assumption $\mathcal{G}$ is connected. Therefore, for any $x \in \mathbb{R}^n$, we have that $\sum_{i=1}^n \sum_{j=1}^n \pi_i p_{ij} (x_i-x_j)^2  = \sum_{(i,j) \in \mathcal{E}_R} \pi_{i} p_{ij} (x_i-x_j)^2.$
 
Observe that in Corollary \ref{cor:maincor}, the optimal value does not change when we multiply the vector $\pi$ by a constant factor. Consequently, we will instead deal with the un-normalized quantities $\hat \pi_i$ defined in (\ref{pidef}), which will lead to less cumbersome expressions. 

We now bound the optimal value of optimization problem for $\lambda_2(P)$ appearing in Corollary \ref{cor:maincor}. Let $x$ be any element of $S$. Without losing generality, we can assume that $x_1<x_2<\dots<x_n$ and assume that $x_n$ denotes the component of $x$ which is largest in magnitude (replace $x$ by $-x$ if this is not true). Consider
\begin{align*}&\sum_{(i,j) \in \mathcal{E}_R} \hat \pi_{i} p_{ij} (x_i-x_j)^2 \\
\overset{(a)}{=}& 2 \sum_{(i,j) \in \mathcal{E}_R,i<j} \hat \pi_{i} p_{ij} (x_j-x_i)^2\\
\overset{(b)}{\geq}& 2 \sum_{(i,j) \in \mathcal{E}_R,i<j} \hat \pi_{i} p_{ij} \sum_{k=i}^{j-1} (x_{k+1}-x_k)^2\\
\overset{(c)}{=}& 2 \sum_{i=1}^{n-1} \sum_{k \leq i,l \geq i+1} {\hat \pi}_k p_{kl} (x_{i+1}-x_i)^2\\
\overset{(d)}{=}& 2 \sum_{i=1}^{n-1} \sum_{k \leq i,l \geq i+1} \alpha_k^4 \sum_{p: k \sim p, l \sim p} \frac{\beta_{p}^2 }{\sum_{q: q \sim p} \alpha_q^2} (x_{i+1}-x_i)^2,
\end{align*}
where (a) follows by reversibility; (b) is because that $x_1<x_2<\dots<x_n$; in (c), we rearrange two summations; we use definitions of $p_{kl}$ and $\hat \pi_k$ in (d), as well as the fact that $x_t \rightarrow c \alpha, y_t \rightarrow c^{-1} \beta$ for some $c$. Using the assumption that $b \leq \alpha_i, \beta_j \leq b^{-1}$ and letting $d(j)$ denote the degree of $j \in \mathcal{V}_C$, we have that
\begin{align}\label{rearragelambda2}
    &\sum_{(i,j) \in \mathcal{E}_R} \hat \pi_{i}  p_{ij} (x_i-x_j)^2 \nonumber\\
    \geq& 2 b^8 \sum_{i=1}^{n-1} \left( \sum_{k \leq i,l \geq i+1}  \sum_{p: k \sim p, l \sim p} \frac{1 }{d(p)}\right) (x_{i+1}-x_i)^2\nonumber\\
    \overset{(e)}{\geq}& b^8 \sum_{i=1}^{n-1} (x_{i+1}-x_i)^2,
\end{align}where (e) is followed by (5) in \cite{olshevsky2013degree} which has proved that on any undirected connected graph, regardless of the node labeling we have that $\sum_{k \leq i,l \geq i+1}  \sum_{p: k \sim p, l \sim p} \frac{1 }{d(p)} \geq \frac{1}{2}.$
   
Next, let $\hat \pi_{max}$ denote the component of $\hat \pi$ which has the largest value. Then $1=\sum_{i=1}^n \hat \pi_i x_i^2 \leq n \hat \pi_{max} x_n^2$ and hence $x_n \geq (n \hat \pi_{max})^{-\frac{1}{2}}$ (because by construction, $x_n$ is the entry of $x$ with the largest value). Since $\sum_{i=1}^n \hat \pi_ix_i=0$, then all components of $x$ cannot be positive, so $x_1<0$. Consequently, $x_n-x_1 >(n \hat \pi_{max})^{-\frac{1}{2}}$. By Cauchy–Schwartz inequality, we have that
\begin{small}
\begin{align*}{}
(n \hat \pi_{max})^{-1} <& (x_n-x_1)^2\\=&[(x_{n}-x_{n-1})+(x_{n-1}-x_{n-2})+\cdots+(x_{2}-x_{1})]^2\\\leq& (n-1) \sum_{i=1}^{n-1} (x_{i+1}-x_{i})^2.
\end{align*} 
\end{small}
Dividing by $n-1$ on both sides and using (\ref{pidef}), it is immediate that
\begin{equation}\label{lambda2part}
    \sum_{i=1}^{n-1} (x_{i+1}-x_{i})^2 > \frac{b^4}{n(n-1)\Delta}.
\end{equation}{}Plugging (\ref{lambda2part}) into (\ref{rearragelambda2}) and by variational characterizations, we have that
\begin{equation}\label{lambda2bound}
    \lambda_2(P) < 1-\frac{b^{12}}{n(n-1)\Delta}.
\end{equation}{}

It remains to obtain an upper bound on $-\lambda_n$. It turns out that this is done in the easiest possible way with the Gershgorin circle theorem. Indeed, for the matrix $P$, recall that its diagonal entries are $[P]_{ii} = \frac{\alpha_i^2}{\sum_{k:i \sim k}\beta_{k}^2} \sum_{l \in N(i)} \frac{\beta_{l}^2}{\sum_{k:k \sim l}\alpha_k^2},$ which are bounded below by $\frac{b^8}{\Delta}$. Also, $P$ is a stochastic matrix and hence $R_{i} := \sum_{i \neq j} p_{i,j}$ is smaller than 1. Gershgorin's theorem asserts that each eigenvalue of $P$ is in at least one of the disks $\{ \lambda :| \lambda-p_{ii}| \leq 1\}$ for $i=1,\cdots,n$, and consequently,
\begin{equation}\label{lambdanbound}
    \lambda_n(P)>-1+\frac{b^8}{\Delta}.
\end{equation}{}

Combining (\ref{lambda2bound}), (\ref{lambdanbound}), and Corollary \ref{cor:maincor}, we obtain $\gamma_{{\rm asym}}(P) \leq \max\{\lambda_2(P),-\lambda_n(P)\}< 1-\frac{b^{12}}{n(n-1)\Delta}<1-\frac{b^{12}}{n^3}.$
\end{proof}
\bigskip
\noindent {\em Remark:} We now discuss the implications of this theorem for convergence rate. Given any initial condition $u_0$, we have that, after a finite period, the convergence rate will be upper bounded by $1- b^{12}/(n (n-1) \Delta)$. It follows that for all large enough $t$, we can bound $||u_t - u^*||_2 \leq  \left(1-\frac{b^{12}}{n(n-1)\Delta} \right)^t ||u_0 - u^*||_2,$ which translates into a time of $O( n (n-1) \Delta \log (1/\epsilon))$ until the distance to $u^*$ shrinks by $\epsilon$. Unfortunately, because $t$ needs to be ``large enough,'' this argument only works if we assume $\epsilon$ is small enough. It is an open question to establish a polynomial convergence time which would hold for all $\epsilon$.

\section{Simulations}
We give simulation results in this section to show that how $1/(1-\gamma_{\rm asym})$ varies with $n$ and $b$ and how tight our bounds are in practice. We perform simulations for five different kinds of graphs $\mathcal{G}$, namely, line, star, 2d-grid and 3d-grid and complete graph. We do 1000 experiments for each $n$, $b$ and graph and in each experiment we generate $2n$ random numbers whose values are between $b$ and $1/b$ to initialize $x_{i,0}, y_{j,0}$ for $i,j \in [n]$. Then we get experimental upper bounds $\eta$ for $1/(1-\gamma_{\rm asym})$.

We first set $b=0.3$ and let $n$ change as we estimate the asymptotic convergence rate from examples in five different kinds of graphs. Figures \ref{line_fix_b} and \ref{others_fix_b} show how the convergence times scale as we increase $n$ for line graph and four different kinds of graphs. All examples Figure \ref{others_fix_b} show sublinear growth, which is consistent with the upper bounds of Theorem \ref{Plambda}, which are always at least quadratic. In these cases, the upper bounds we have derived is conservative. However, our results in Figure \ref{line_fix_b}  appear to grow quadratically in $n$, which suggests that on the line graph the upper bound of our main result is tight up to constant factors.

\begin{figure}[!t]
     \centering{\includegraphics[width=\columnwidth]{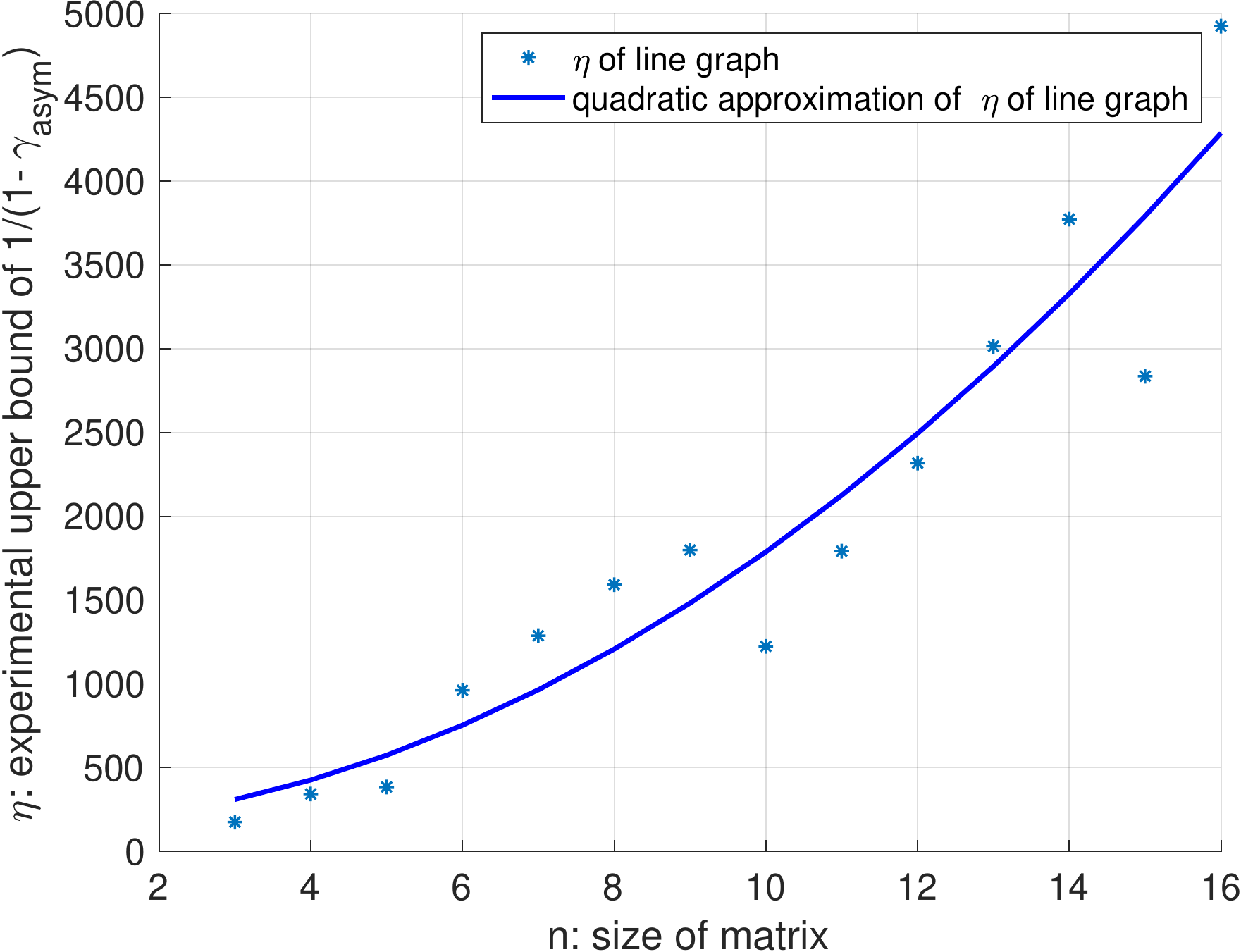}}
     \caption{The experimental upper bounds of line graph when $b=0.3$}
     \label{line_fix_b}
\end{figure}

\begin{figure}[!t]
     \centering{\includegraphics[width=\columnwidth]{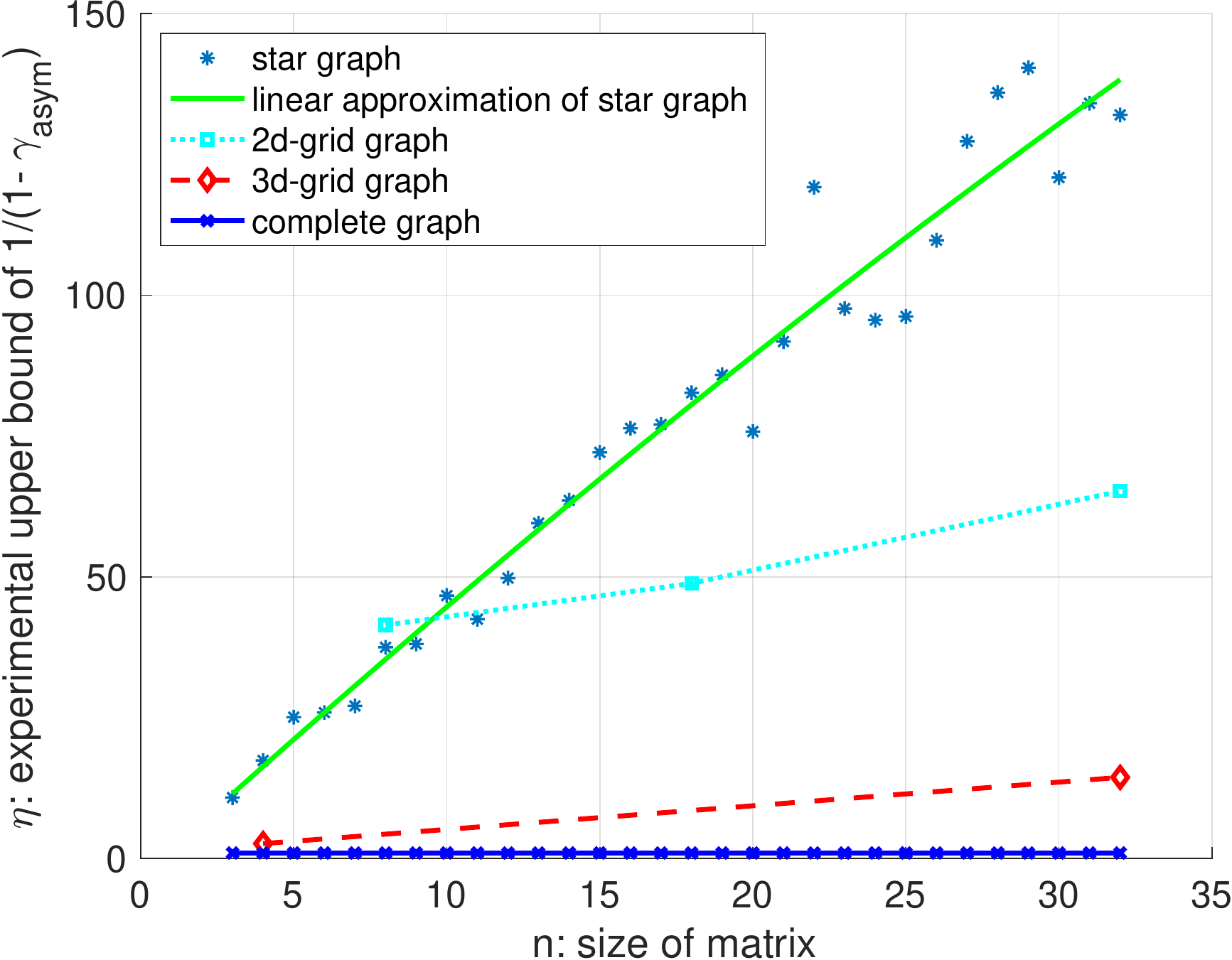}}
     \caption{The experimental upper bounds of graphs when $b=0.3$}
     \label{others_fix_b}
\end{figure}

In Figure \ref{line_fix_n} and Figure \ref{others_fix_n}, we instead fix $n=32$ and let $b$ changes from 0.01 to 1. The results show that indeed it takes more time to converge when $b$ becomes smaller. Our main result also has such a scaling with $b$.  

\begin{figure}[!t]
     \centering{\includegraphics[width=\columnwidth]{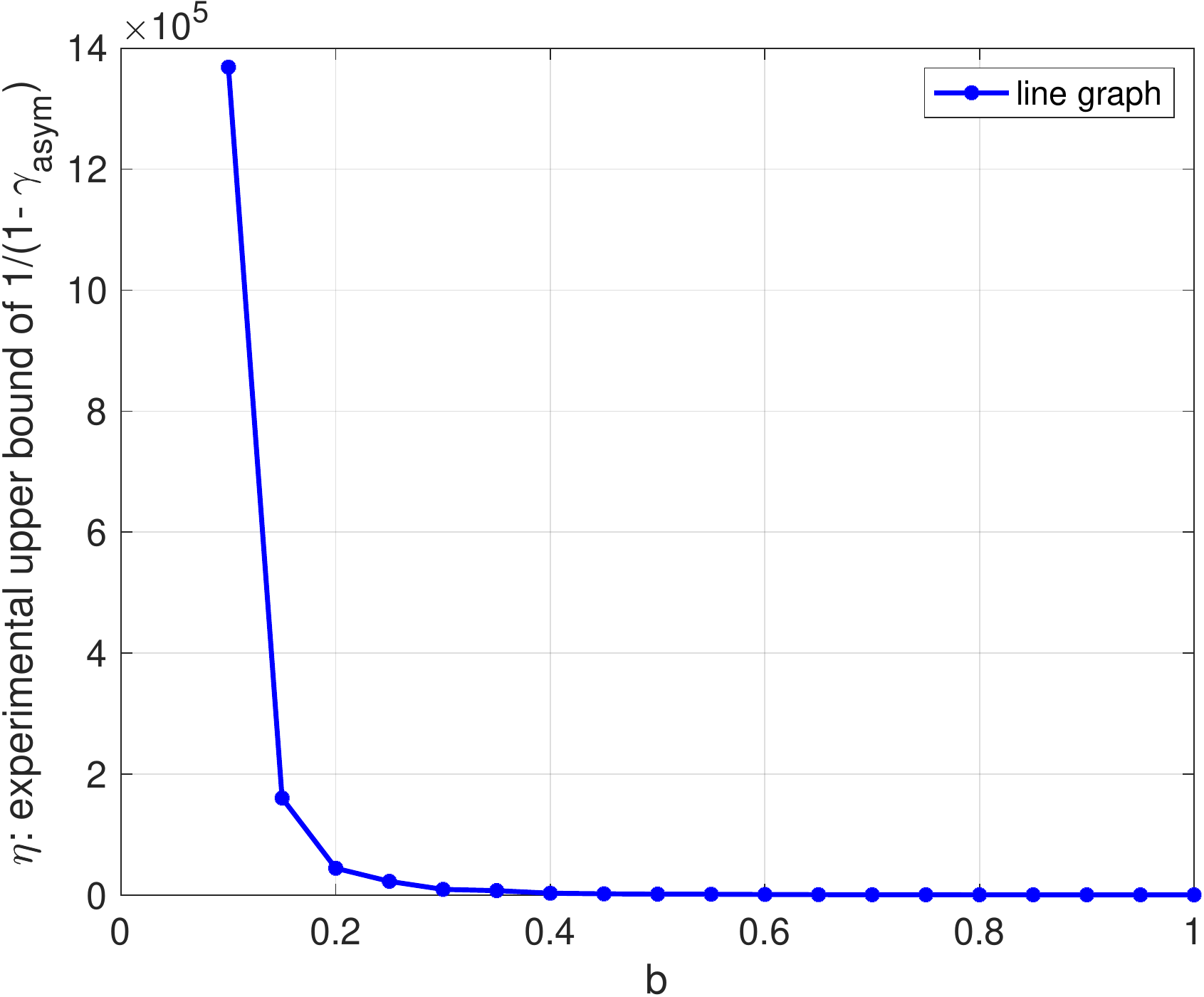}}
     \caption{The experimental upper bounds of line graph when $n=32$}
     \label{line_fix_n}
\end{figure}

\begin{figure}[!t]
     \centering{\includegraphics[width=\columnwidth]{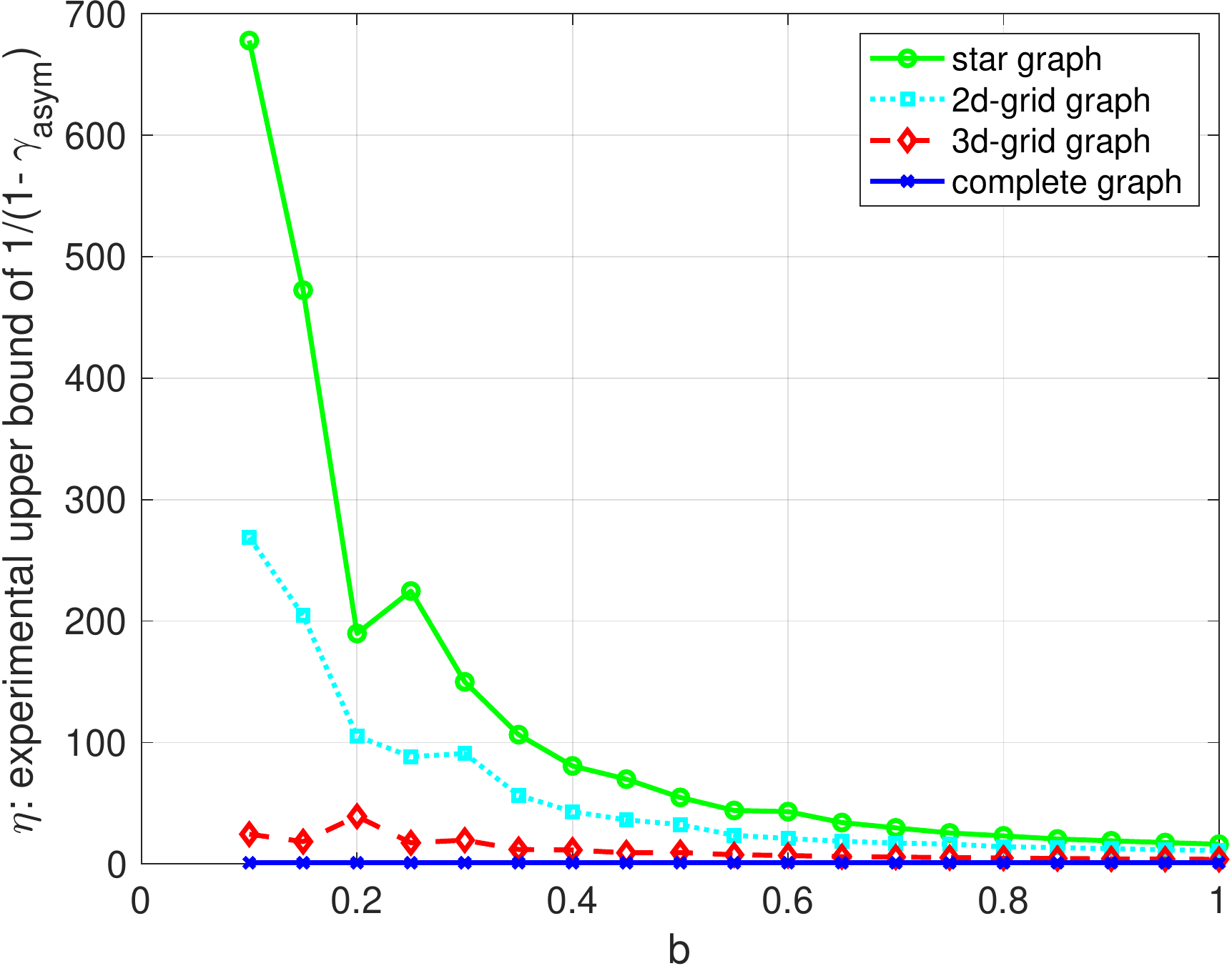}}
     \caption{The experimental upper bounds of graphs when $n=32$}
     \label{others_fix_n}
\end{figure}

In summary, our numerical results suggest that on the line graphs our results are tight, whereas on grids and star graphs, our estimate of the convergence time is conservative. Additionally, simulations suggest the blowup in convergence time as $b \rightarrow 0$ is not merely a feature of our main theorem but also happens in practice. 

\section{CONCLUSIONS}

Our main result has been a derivation of a polynomial bound on the asymptotic convergence rate of alternating minimization for rank-one matrix completion. The main open question left by our work is whether a polynomial convergence time can be proven. This will require a non-asymptotic analysis of the dynamics described here. This appears challenging, as equation (\ref{u&v}) is essentially a switched linear system, and there is no obvious Lyapunov function which would lead to a polynomial rate. 

Furthermore, the approach provides a way to begin analyzing the general case of higher rank matrix completion, i.e., we can write alternating minimization as a consensus problem, even for higher-rank matrices. However, the problem is that the coefficients of this linear combination are time-varying, depending on the current iterate, and might be negative.

\bibliographystyle{IEEEtran}  
\bibliography{IEEEabrv,references}
\end{document}